\documentclass[runningheads]{llncs}
\usepackage{graphicx}
\usepackage{amsmath}
\usepackage{comment}
\usepackage{amssymb}
\usepackage{soul}
\usepackage{xcolor}
\usepackage[T1]{fontenc}

\usepackage{graphicx}
\begin{document}
\title{Bi-View Embedding Fusion: A Hybrid Learning Approach for Knowledge Graph's Nodes Classification Addressing Problems\\with Limited Data}
%
%
\author{Rosario Napoli\inst{1}\orcidID{0009-0006-2760-9889}
\and Giovanni Lonia\inst{1}\orcidID{0009-0000-1824-0123}
\and Antonio Celesti \inst{1}\orcidID{0000-0001-9003-6194} \and Massimo Villari \inst{1}\orcidID{0000-0001-9457-0677} \and Maria Fazio \inst{1}\orcidID{0000-0003-3574-1848} }
\authorrunning{F. Author et al.}
%

%
\authorrunning{R. Napoli}
%
\institute{Department of Mathematical and Computer Sciences, Physical Sciences and Earth Sciences, University of Messina, Viale Ferdinando Stagno d'Alcontres 31 98166, Messina, Italy.
}
\maketitle              
\begin{abstract}
Traditional Machine Learning (ML) methods require large amounts of data to perform well, limiting their applicability in sparse or incomplete scenarios and forcing the usage of additional synthetic data to improve the model training. To overcome this challenge, the research community is looking more and more at Graph Machine Learning (GML) as it offers a powerful alternative by using relationships within data. However, this method also faces limitations, particularly when dealing with Knowledge Graphs (KGs), which can hide huge information due to their semantic nature.
This study introduces Bi-View, a novel hybrid approach that increases the informative content of node features in KGs to generate enhanced Graph Embeddings (GEs) that are used to improve GML models without relying on additional synthetic data. The proposed work combines two complementary GE techniques: Node2Vec, which captures structural patterns through unsupervised random walks, and GraphSAGE, which aggregates neighbourhood information in a supervised way. 
Node2Vec embeddings are first computed to represent the graph topology, and node features are then enriched with centrality-based metrics, which are used as input for the GraphSAGE model. Moreover, a fusion layer combines the original Node2Vec embeddings with the GraphSAGE-influenced representations, resulting in a dual-perspective embedding space. Such a fusion captures both topological and semantic properties of the graph, enabling the model to exploit informative features that may exist in the dataset but that are not explicitly represented. Our approach improves downstream task performance, especially in scenarios with poor initial features, giving the basis for more accurate and precise KG-enanched GML models.

\end{abstract}

\section{Introduction}

In recent years, Machine Learning (ML) has become the dominant approach for solving a wide range of problems across various domains. However, ML methods operate under the assumption that data points are independent and require extensive data to perform well \cite{b72}. These constraints introduce significant limitations, especially in scenarios where data and features are incomplete or expensive to obtain, forcing data scientists to use different techniques, such as increasing the dataset with additional synthetic data, to improve models' training at the cost of using artificial information. To address these limitations, Graph Machine Learning (GML) has emerged as a powerful method for extracting meaningful insights using the inherent relationships in data, representing them in a graph-based format where nodes are entities and edges are relationships. GML models use node features for probabilistic or aggregation strategies \cite{b92} to create Graph Embeddings (GEs), which represent features in continuous vector spaces that are given as input to ML models. 

At the same time, the availability and semantics of interconnected data have increased the widespread adoption of Knowledge Graphs (KGs), which extend basic graph structures by allowing elements to have different labels and properties \cite{b12}, providing a powerful and intuitive way to describe real-world dynamics. This means that KGs-based GML models can capture both structural and semantic information for node classification and link prediction tasks \cite{b53}. However, even with good results, poor datasets are still complex to be analysed for GML tasks, particularly within KGs, where information can be hidden because of their semantic nature. 

To overcome this critical limitation and make the most of the available features without using additional synthetic data, we propose a new hybrid approach that combines unsupervised probabilistic topological embeddings from Node2Vec and centrality metrics with supervised message-passing via GraphSAGE. In particular, Node2Vec is employed to capture global structural patterns via random walks that are combined with classical graph-theoretic centrality measures to create an initial embedding vector. GraphSAGE is then applied using the initial embeddings to generate a new set of GEs that incorporate both local neighbourhood aggregation and the previously derived global structural information. The key innovation in our approach lies in the introduction of a learnable fusion mechanism that adaptively weights the unsupervised (Node2Vec) and supervised (GraphSAGE+Centrality Measures) representations for each node, generating new embeddings that can capture both topological properties and task-specific semantic patterns. This dual integration strategy improves node classification in poor datasets and enables a deeper analysis of the available hidden knowledge within KGs. Our experiments 
confirmed the significant effectiveness of this approach, demonstrating consistent improvements over standalone node classification baselines.

The remainder of this paper is organised as follows. Section \ref{sec:definition} explains the basic concepts on which our scientific work is based. Section \ref{sec:related_work} discusses the related works in the context of node classification on KGs. Section \ref{sec:bi-view_strategy} introduces our Bi-View hybrid embedding strategy. Sections \ref{sec:model_validation} and \ref{sec:sp_model} validate the structural enrichment process and present the supervised fusion model. Sections \ref{sec:experiments} and \ref{sec:setup_and_model_validation} describe the experimental setup and evaluation results. Finally, Section \ref{sec:conclusion} concludes the paper and outlines future research directions.

\section{Definitions}
\label{sec:definition}
This Section introduces the formal definitions of the core concepts used throughout this study.

\begin{definition}[Knowledge Graph]
A Knowledge Graph is defined as a tuple \( KG = (V, R, E, \ell_V) \), where:
\begin{itemize}
    \item \( V \) is the set of entities (nodes);
    \item \( R \) is the set of relationship types;
    \item \( E \subseteq V \times R \times V \) is the set of labelled, directed edges (triplets);
    \item \( \ell_V : V \rightarrow 2^C \) is a labelling function that assigns one or more class labels from the set \( C \) to each node.
\end{itemize}

\end{definition}

\begin{definition}[Graph Embedding]
Let \( KG = (V, E, R,  \ell_V) \) be a Knowledge Graph. A Graph Embedding (GE) is a function
\[
f: V \rightarrow \mathbb{R}^d,
\]
where \( d \ll |V| \), which maps each node \( v \in V \) to a vector \( f(v) \in \mathbb{R}^d \).
\end{definition}

\begin{definition}[Transition probability]
Given a Knowledge Graph \( KG = (V, R, E, \ell_V) \), let \( c = (v_1, v_2, \ldots, v_k) \) be a path. The transition probability of moving from node \( v = c_{i-1} \) to a neighbor \( x = c_i \) is defined as
\begin{equation}
P(c_i = x \mid c_{i-1} = v) = \pi_{xv} = \frac{\alpha_{pq}(t,x) \cdot w_{vx}}{Z},
\end{equation}
\noindent
where \(w_{vx}\) is the cumulative weight of the edges between nodes \(v\) and \(x\), \(Z\) is a normalization constant and \(\alpha_{pq}(t,x)\) is a search bias defined as
\[
\alpha_{pq}(t,x) =
\begin{cases}
\frac{1}{p} & \text{if } d_{tx} = 0 \ , \\
1           & \text{if } d_{tx} = 1 \,, \\
\frac{1}{q} & \text{if } d_{tx} = 2 \,
\end{cases}
\]
\noindent
where \(d_{tx}\) is the shortest path distance between node \(t\) (the node visited before \(v\)) and node \(x\), \(p\) is the return parameter and \(q\) is the in-out parameter.
\end{definition}
\begin{definition}[Node2Vec on Knowledge Graphs as Unsupervised Aggregator]
Given a Knowledge Graph \( KG = (V, R, E, \ell_V) \), Node2Vec \cite{b78} is a method that learns low-dimensional vector embeddings \( \phi : V \rightarrow \mathbb{R}^d \) for nodes by simulating biased random walks over the graph structure.
A sequence of nodes is generated via random walks defined by the transition probability:
\begin{equation}
 P(c_{i} = x \mid c_{i-1} = v) =
\begin{cases}
\frac{1}{p} & \text{if } x = c_{i-2} \\
1           & \text{if } d(x, c_{i-2}) = 1 \\
\frac{1}{q} & \text{otherwise}
\end{cases}   
\end{equation}
where:
1) \( c_i \) is the node \( c \) at position \( i \) in the walk; 
2) \( p \) is the return parameter controlling the likelihood of revisiting the previous node; 
3) \( q \) is the in-out parameter controlling the likelihood of exploring further away nodes; 
4) \( d(x, c_{i-2}) \) is the shortest path distance between node \( x \) and the node before the current one.

\end{definition}

\begin{definition}[GraphSAGE on Knowledge Graphs as Supervised Aggregator]
Given a Knowledge Graph \( KG = (V, R, E, \ell_V) \), GraphSAGE (Graph Sample and Aggregate) is an inductive framework that learns
embeddings
by aggregating information from the local neighbourhood of each node \cite{b77}.
At each layer \( k \), the representation of a node \( v \in V \), denoted as \( h_v^{(k)} \in \mathbb{R}^d \), is updated as
\begin{equation}
h_v^{(k)} = \sigma\Big( W^{(k)} \cdot 
\text{AGG}^{(k)}\big( \{ h_u^{(k-1)} \mid (u, r, v) \in E \} 
\cup \{ h_v^{(k-1)} \} \big) \Big)
\end{equation}
where:
1) \( h_v^{(0)} \) is the initial feature vector of node \( v \); 
2) \( \text{AGG}^{(k)} \) is a permutation-invariant aggregator function; 
3) \( W^{(k)} \) is a trainable weight matrix at layer \( k \); 
4) \( \sigma \) is a non-linear activation function.
\end{definition}

\section{Related Works on Classification in KGs Datasets} \label{sec:related_work}

In recent years, as relational datasets have become increasingly prevalent across domains such as social networks and biological systems~\cite{b79}, the need to extract meaningful insights from limited data has become mandatory~\cite{b94}. In this context, node classification has emerged as a core ML task that aims to assign labels to nodes based on structural and semantic information. However, challenges increase when data are highly semantic, as graph-based learning models heavily depend on explicated features to perform well. In such settings, using both intrinsic topology and node-level features is essential to increase prediction accuracy. This issue is particularly evident in KGs, where, despite the apparent availability of information, a high degree of semantic may lead to unexpressed knowledge and underconnected graph structures ~\cite{b97}.
To address this, Graph Embeddings (GEs) have emerged as a fundamental tool in the node classification pipeline~\cite{b74}. By translating both the structural and semantic properties of graphs into low-dimensional vector spaces, GEs enable the application of conventional ML algorithms to complex graph data~\cite{b104} using random walk-based methods and neural network-based methods~\cite{b103}.
However, when the semantics are not fully explicated, models rely only on either the graph structure or node attributes, tending to underperform~\cite{b101}.
As a consequence, the field remains relatively new, and determining how to enhance data expressiveness for downstream node classification tasks remains an important challenge. Ongoing research in GML has investigated multiple strategies aimed at improving the quality and expressiveness of embeddings.

Recent efforts have led to the development of different GEs learning approaches designed to bridge the gap between sparse datasets and the rich structural information inherent in KGs. Single-view embeddings learning like unsupervised models, such as Node2Vec, extract latent topological information using stochastic walks, while supervised methods, such as GraphSAGE, exploit node features through neighborhood aggregation, but both approaches individually struggle to capture the full spectrum of relationships and semantic patterns encoded in KGs, especially in low-feature settings~\cite{b100}. 

In response to the limitations of single-view embeddings, research has also embraced multi-view fusion strategies~\cite{b95}. These techniques aim to enhance the representation quality by capturing different perspectives of the same graph. The most explored strategies are: concatenation, which merges multiple embeddings to increase coverage; attention-based mechanisms, which assign context-aware importance to each embedding type; and autoencoders, which project different embeddings into a shared latent space, thereby enabling compression of redundant information~\cite{b94}. 
In the end, meta-inductive models, rather than training on a single fixed graph, acquire transferable temporal meta-knowledge capable of generalising across domains at different time~\cite{b98}. 

Despite the breadth of methodologies developed to improve the quality and effectiveness of GEs and the growing attention toward strategies that address feature scarcity, current approaches do not explicitly explore the fusion of Node2Vec with centrality features as an auxiliary input to GraphSAGE to create GEs enhanced representations that directly target the limitations typical of small-feature datasets. This gap underscores the need for a novel hybrid strategy capable of capturing both structural behaviour and semantics in scenarios characterised by limited node availability.

To the best of our knowledge, no prior study has investigated this specific hybridisation pipeline, where unsupervised structural embeddings and centrality measures are explicitly combined with supervised feature-aware representations in an adaptive way within the context of low-features graph scenarios. This represents a promising and innovative approach to address the challenges of embedding expressiveness under dataset constraints.

\section{Bi-View Strategy} \label{sec:bi-view_strategy}
To address classification problems in poor-features datasets within KGs, we propose a hybrid strategy that merges global topological structure with local semantic information. Our proposed approach (Figure~\ref{fig:biviewarch}), is composed of five main components. In particular, the process begins with the Node2Vec Encoder, which captures the structural properties of the graph by generating node global connectivity embeddings $Z^{(n2v)}$ through biased random walks. In parallel, the Centrality Vector module computes a set of graph-theoretic metrics $\gamma(V)$, including measures such as PageRank and Betweenness, to quantify the importance of each node within the overall structure.

Both the structural embeddings and the centrality metrics are then combined in the Feature Aggregator, which builds enriched node representations $F_{\text{agg}} = \mathcal{A}(\gamma(V), Z^{(n2v)})$ that serve as input to the next component. The GraphSAGE Aggregator applies supervised neighborhood aggregation over these features to learn new embeddings $Z^{(sage)}$,  integrating local semantic information with the previously extracted global features.

Finally, the Embedding Fusion module adaptively combines the two views ($Z^{(n2v)}$ and $Z^{(sage)}$) through a learnable mechanism that balances their contributions based on node-specific characteristics. The result is a fused representation $Z^{(fused)}$ for each node, which we refer to as the Bi-View embedding. This representation captures both the structural roles and semantic contexts of nodes in new embeddings, improving classification tasks in 
poor-features KGs scenarios.

\subsection{Node2Vec Encoder}

Let \( Z^{(n2v)} \in \mathbb{R}^{n \times d_1} \) be the matrix of Node2Vec embeddings, where each row \( \varphi(v_i) \in \mathbb{R}^{d_1} \) is the embedding of node \( v_i \) generated by making biased random walks over the graph. These embeddings represent topological similarity, capturing information depending on the tuning of the return parameter \( p \) and the in-out parameter \( q \).

\subsection{Feature Aggregator}

To increase the number of features without adding any artificial data, we add a set of classical centrality metrics that are calculated for each node. Formally,  let \( \gamma(v_i) \in \mathbb{R}^{d_2} \) be the centrality-based features vector for a node \( v_i \), including measures such as degree, betweenness, and PageRank. These topological metrics are concatenated with the Node2Vec embeddings in order to create an enriched initial feature vector:
\[
h^{(0)}_i = \left[ \ell_V(v_i) \, \| \, \varphi(v_i) \, \| \, \gamma(v_i) \right] \in \mathbb{R}^{d_0},
\]
where \( \ell_V(v_i) \) is the attribute-based features of node \( v_i \), and \( \| \) the vector concatenation.

\subsection{GraphSAGE Aggregator}

GraphSAGE is then applied by using \( h^{(0)}_i \) as input features. In particular, at each layer \( k \), the node embeddings are updated as follows:
\[
h^{(k)}_i = \sigma \left( W^{(k)} \cdot \text{AGG}^{(k)} \left( \{ h^{(k-1)}_j \mid j \in \mathcal{N}(i) \} \cup \{ h^{(k-1)}_i \} \right) \right),
\]
where \( \text{AGG}^{(k)} \) is a permutation-invariant aggregator, \( W^{(k)} \) is a learnable weight matrix, and \( \sigma \) is a non-linear activation function. The output embeddings of this step are \( Z^{(sage)} \in \mathbb{R}^{n \times d_2} \).

\subsection{Embedding Fusion}

A dynamic fusion mechanism that learns node-specific main characteristics to generate enhanced GEs representations is then applied. In particular, for each node \( i \), the fusion coefficient \( \alpha_i \in (0,1) \) is calculated as:
\[
\alpha_i = \sigma \left( W_\alpha \cdot \left[ z^{(n2v)}_i \, \| \, z^{(sage)}_i \right] + b_\alpha \right),
\]
where \( \sigma \) is the sigmoid function and \( W_\alpha \in \mathbb{R}^{1 \times (d_1 + d_2)} \), \( b_\alpha \in \mathbb{R} \) are learnable parameters.

The final fused embedding is given by:
\[
z^{(fused)}_i = \alpha_i \cdot z^{(n2v)}_i + (1 - \alpha_i) \cdot z^{(sage)}_i.
\]
Allowing model balancing between structural and semantic information per node.

\subsection{Centrality Vector}

This component computes a set of classical centrality metrics for each node, giving deterministic structural features. In particular, PageRank calculates node importance based on the probability of arriving at a node via random walks, while Betweenness Centrality measures how frequently a node appears on the shortest paths between other node pairs, representing its control over information flow in the network.

Formally, for each node $v \in V$, the centrality vector is defined as:
\[
\gamma(v) = [\mathrm{deg}(v), \mathrm{PageRank}(v), \mathrm{Betweenness}(v), \ldots] \in \mathbb{R}^{d_2}
\]
These features enrich the node representation with structural insights that are independent from the aggregation process.

\begin{figure}
    \centering
    \includegraphics[width=0.8\linewidth]{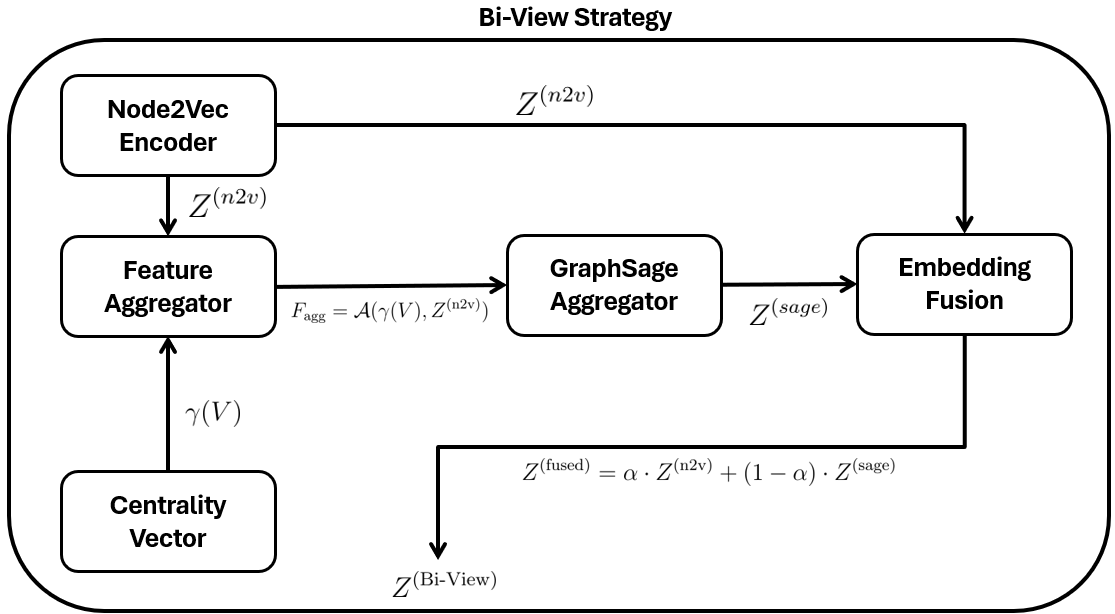}
    \caption{Bi-View strategy.}
    \label{fig:biviewarch}
\end{figure}

\section{Model Validation of Structural Enrichment in GraphSAGE via Node2Vec Embeddings and Centrality Metrics} \label{sec:model_validation}

In this Section, we provide a formal validation to describe how structural embeddings generated by Node2Vec and centrality features affect the behaviour of the GraphSAGE aggregation process. Our goal is to model the transition from topological unsupervised learning to a semantically enriched representation, where structural information derived from random walks and graph-theoretic properties is used to boost the learning process.

\begin{definition}{Node2Vec as a Generator of Structural Features}

Let \( KG = (V, R, E, \ell_V) \) be a KG. Node2Vec \cite{b78} is based on an embedding function:
\[
\phi : V \rightarrow \mathbb{R}^{d_1},
\]

Where the embedding \( \phi(v) \) captures the topology of a node \( v \) within the network, interdependently of any node label or attribute.
\end{definition}

\begin{definition}[Centrality Vector]
Let \( \gamma : V \rightarrow \mathbb{R}^{d_2} \) be a function that maps each node \( v \in V \) to a vector of centrality metrics.
\[
\gamma(v) = \left[ \text{deg}(v), \text{PageRank}(v), \text{Betweenness}(v), \ldots \right],
\]
\end{definition}

These values give a deterministic perspective, acting as global topological analysis that is independent from random walks.

\begin{definition}[Structurally Enriched Initial Features]
We define the initial features vector that is given to the GraphSAGE aggregator  for each node \( v \in V \) as:
\[
h^{(0)}_v = [\ell_V(v) \, \| \, \phi(v) \, \| \, \gamma(v)] \in \mathbb{R}^{d_0},
\]
where \( \ell_V(v) \) is the original attribute-based feature, \( \phi(v) \) is the Node2Vec embedding and \( \gamma(v) \) is the centralities vector.
\end{definition}

\begin{definition}[GraphSAGE with Structural Feature Injection]
Let \( h^{(0)}_v \in \mathbb{R}^{d_0} \) be an enriched feature vector. At each layer \( k \), GraphSAGE updates the node embeddings using the following rule:
\[
h^{(k)}_v = \sigma \left( W^{(k)} \cdot \text{AGG}^{(k)}\left( \{ h^{(k-1)}_u \mid u \in \mathcal{N}(v) \} \cup \{ h^{(k-1)}_v \} \right) \right),
\]
\end{definition}

\begin{lemma}[Structural Enrichment Increases Neighborhood Discriminability]
Let \( h^{(0)}_v \) and \( h^{(0)}_u \) be the enhanced GEs of two distinct nodes \( v, u \in V \) and \( \ell_V(v) \) and \( \ell_V(u) \) the not enhanced GEs. Then, if \( \phi(v) \neq \phi(u) \) or \( \gamma(v) \neq \gamma(u) \), it follows that:
\[
\| h^{(0)}_v - h^{(0)}_u \|_2 \geq \| \ell_V(v) - \ell_V(u) \|_2.
\]
\end{lemma}

\begin{proof}
Since the enriched feature vector includes additional components, in form of:
\[
h^{(0)}_v = [\ell_V(v) \, \| \, \phi(v) \, \| \, \gamma(v)],
\]
Any difference in \( \phi \) or \( \gamma \) will add non-negative contributions to the Euclidean distance. By the Pythagorean theorem:
\[
\| h^{(0)}_v - h^{(0)}_u \|_2^2 = \| \ell_V(v) - \ell_V(u) \|_2^2 + \| \phi(v) - \phi(u) \|_2^2 + \| \gamma(v) - \gamma(u) \|_2^2.
\]
Thus, the distance is always greater than or equal to that computed using initial attributes.
\end{proof}

\begin{definition}[Enriched Neighborhood Aggregation Influence]
Let \( \mathcal{N}(v) \) be the set of neighbors of node \( v \), with \( h^{(0)}_v \) the enriched GE. The influence score of node \( v \) at layer \( k \), denoted \( I_k(v) \), is defined as
\[
I_k(v) = \frac{1}{|\mathcal{N}(v)|} \sum_{u \in \mathcal{N}(v)} \| h^{(k-1)}_v - h^{(k-1)}_u \|_2.
\]
\end{definition}

\begin{lemma}[Structural Enrichment Amplifies Neighborhood Gradient]
Let \( I_k(v) \) be the influence score of node \( v \) under standard GraphSAGE, and \( \tilde{I}_k(v) \) be the same score using structurally enriched features. Then:
\[
\tilde{I}_k(v) \geq I_k(v),
\]
Assuming that \( \phi(u) \) and \( \gamma(u) \) differ across neighborhoods.
\end{lemma}

\begin{proof}
Similarly to the previous lemma, the added feature space components from \( \phi \) and \( \gamma \) increase differences across the neighbourhood, increasing the mean pairwise distances and enhancing node representations with global structural embeddings and topological information, in particular:
1) improves the initial separability of nodes in the embedding space; 
2) enhances the information gradient available during message passing; 
3) facilitates better generalisation under low-data regimes by reducing feature sparsity.

\end{proof}

\section{Supervised Fusion Model for Enhanced Bi-View Node Representations} \label{sec:sp_model}
To capture the best specific characteristics of both Node2Vec and GraphSAGE embeddings, we propose a supervised fusion strategy that learns the best optimised representation for node classification. 
Let $z_{i}^{(n2v)}\in R^{d_{1}}$ be the unsupervised embedding of node $i$ produced by Node2Vec, and $z_{i}^{(sage)}\in R^{d_{2}}$ the supervised embedding generated by GraphSAGE, which uses node2vec and centrality features. We concatenate both embeddings into a single vector:
$$z_{i}^{(cat)}=[z_{i}^{(n2v)}||z_{i}^{(sage)}]\in R^{d_{1}+d_{2}}$$
This concatenated representation is then passed through a learnable encoder function $f_{\theta}:R^{d_{1}+d_{2}}\rightarrow R^{d}$ implemented as a multi-layer perceptron (MLP), which learns the best task-specific embedding:
$$z_{i}^{(enh)}=f_{\theta}(z_{i}^{(cat)})$$
The resulting embedding $z_{i}^{(enh)}\in R^{d},$ with $d\ll d_{1}+d_{2}$ is used for classification via a function:
$$\hat{y}_{i}=softmax(W_{c}\cdot z_{i}^{(enh)}+b_{c})$$
where $W_{c}\in R^{|\mathcal{Y}|\times d}$ and $b_{c}\in R^{|\mathcal{Y}|}$ are trainable parameters. In the end, the full model is trained over the labelled node set $\mathcal{L}$:
$$\mathcal{L}_{enh}=-\sum_{i\in\mathcal{L}}y_{i}\cdot log\hat{y}_{i}$$
This supervised fusion mechanism allows the model to automatically extract the most important aspects of both global (Node2Vec) and local (GraphSAGE) node contexts, resulting in a topologically and semantically enriched embedding space.

\begin{lemma}[Expressiveness of Fused Embeddings]
    Let $z_{i}^{(n2v)}\in R^{d_{1}}$ and $z_{i}^{(sage)}\in R^{d_{2}}$ be the unsupervised and supervised embeddings of node $i$, respectively. Let $f_{\theta}:R^{d_{1}+d_{2}}\rightarrow R^{d}$ be a Multi-Layer Perceptron (MLP) our fusion function, where $d \le d_1+d_2$. The resulting fused embedding is $z_{i}^{(enh)} = f_{\theta}([z_{i}^{(n2v)}||z_{i}^{(sage)}])$. Given that $f_{\theta}$ is a non-linear function, the fused embeddings $Z^{(enh)} = \{z_{i}^{(enh)} | i \in V\}$ can capture complex interactions between the two original embedding spaces.
\end{lemma}

\begin{proof}
    The concatenation $[z_{i}^{(n2v)}||z_{i}^{(sage)}]$ results in a feature vector with a dimension equivalent to the Cartesian product of the individual embedding spaces, i.e., $R^{d_1} \times R^{d_2}$. The dimension of this concatenated space is $d_1+d_2$.

The multi-layer perceptron $f_{\theta}$ is applied to the concatenation of Node2Vec and GraphSAGE embeddings, allowing for the combination of heterogeneous information into a lower-dimensional space to generate better feature descriptions for classification tasks. In particular, the fused embedding space $Z^{(enh)}$ is capable of representing a richer set of patterns and relationships than either $\mathcal{Z}_{n2v}$ or $\mathcal{Z}_{sage}$ individually.
\end{proof}

\section{Experimental Evaluation on Food and Drug Administration Adverse Effects Dataset} \label{sec:experiments}
\subsection{Dataset Description}
To assess the impact of our framework, we applied it to a real-world dataset derived from the United States Food and Drug Administration (FDA)  Adverse Event Reporting System (FAERS)\footnote{https://open.fda.gov/data/faers/}. FAERS is a database developed by the FDA to support post-marketing safety surveillance of approved drugs. In particular, it collects reports of adverse events and medication errors submitted by healthcare professionals, consumers and manufacturers. 
The dataset was formalised as a KG\footnote{https://github.com/neo4j-graph-examples/healthcare-analytics/tree/main/data}. with classes: Case, Drug, Reaction, Outcome, ReportSource, Therapy and Age Group (Fig. \ref{fig:faers_schema}) that are highly imbalanced (Fig. \ref{fig:class_distribution}).

\begin{figure}[ht]
    \centering
    \begin{minipage}[b]{0.46\linewidth}
        \centering
        \includegraphics[width=\linewidth]{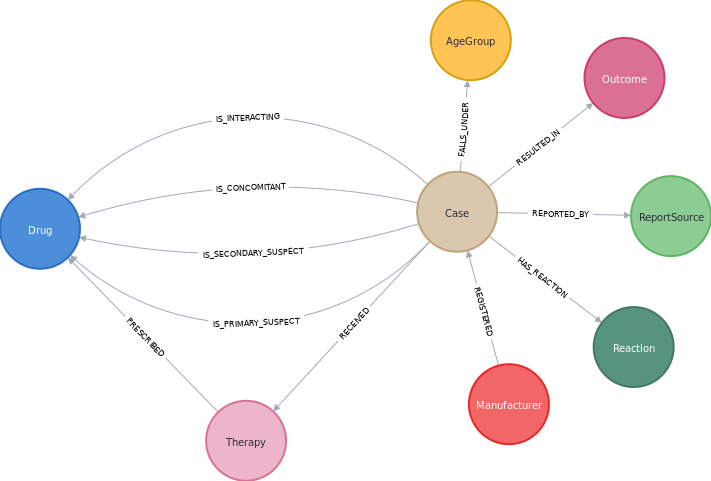}
        \caption{KG schema for FAERS dataset.}
        \label{fig:faers_schema}
    \end{minipage}
    \hfill
    \begin{minipage}[b]{0.50\linewidth}
        \centering
        \includegraphics[width=\linewidth]{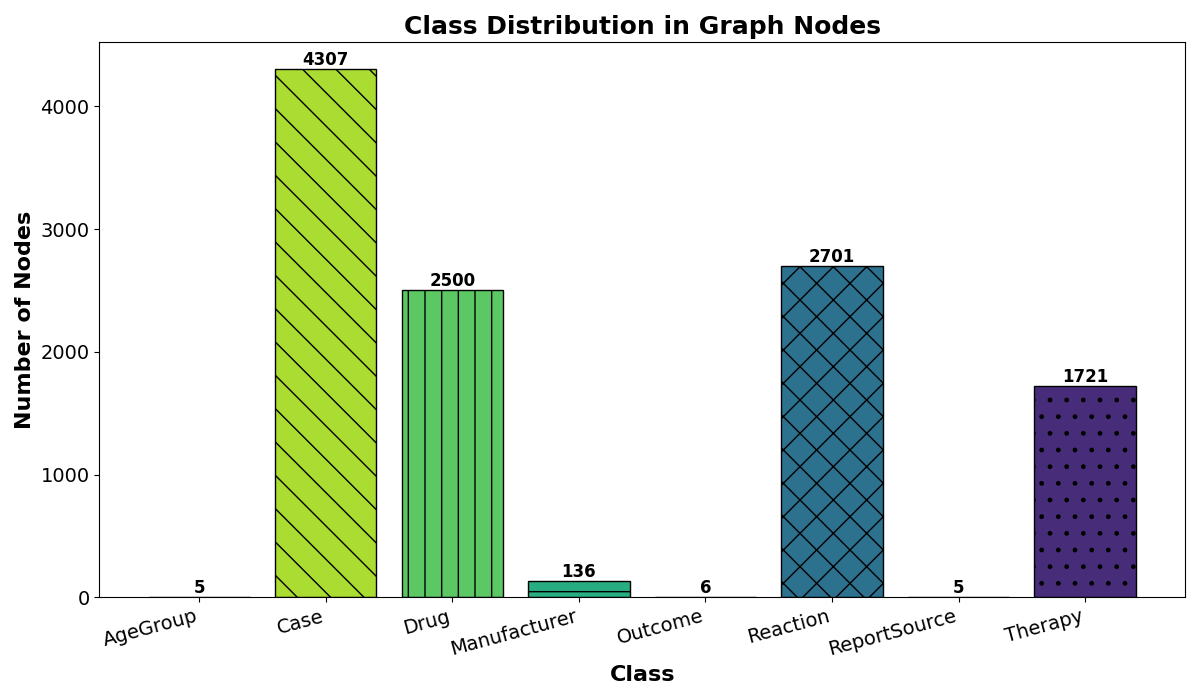}
        \caption{Class distribution.}
        \label{fig:class_distribution}
    \end{minipage}
\end{figure}

\section{Experimental Setup and Model Evaluation}\label{sec:setup_and_model_validation}
We carried out experiments using Neo4j 5.24.0\footnote{\url{https://neo4j.com}}, with the Graph Data Science (GDS) library version 2.12.0\footnote{\url{https://neo4j.com/docs/graph-data-science}}. Starting from this small real-world dataset (\textasciitilde11{,}000 elements), our goal was to train a classifier capable of identifying the most frequent node classes, which are: Patient, Drug, Therapy and Manufacturer. 
For topological embeddings, the Node2Vec algorithm was used with a dimension of 64 elements, a walk length of 80, and 10 iterations, setting both the return and in-out parameters to 1.0. GraphSAGE was then employed with a dimension of 64, using as input features: the centrality properties (betweenness and pageRank), the class label, and the results coming from Node2Vec. GEs were combined using a supervised FusionNet model, composed of an MLP with a hidden dimension of 128 and an output dimension of 64, trained with cross-entropy loss over 100 epochs. In the end, a decision tree was implemented to make classification comparisons among the experiments.
The code with hyperparameter configurations used for the entire pipeline will be made public upon acceptance.

Initially, we generated feature vectors containing only node class labels, with PageRank and Betweenness Centrality as additional features. Finally, we compared the classifier’s performance under three different configurations: (1) using only Node2Vec embeddings, (2) using GraphSAGE embeddings, and (3) using the fused embeddings generated by our strategy (Bi-View).

The network embeddings visualisation was done with Principal Component Analysis (PCA), which is a dimensionality reduction algorithm that allows to understand the quality of latent representations generated by models. In particular, this section also presents a comparative analysis of 2D embeddings obtained from the three distinct models, in order to visually evaluate inter-class separability within the latent space.
From experiments, Node2Vec (Fig. 4) performs the weakest in terms of embedding clarity, with extensive class overlap and poor separability, including for Classes AgeGroup and Case. GraphSAGE (Fig. 5), in contrast, arranges its embeddings along a smooth crescent-shaped curve, lacking compact clusters but possibly capturing latent class semantic relationships in a one-dimensional progression, with its curvilinear structure complicating classification using standard decision boundaries. In the end, Bi-View (Fig. 6) exhibits a strong ability to form well-separated and compact clusters, particularly for class Case, and moderately for Classes AgeGroup and Outcome, aligning with its high classification performance.  

\begin{figure}[ht]
    \centering
    \begin{minipage}[b]{0.32\linewidth}
        \centering
        \includegraphics[width=\linewidth]{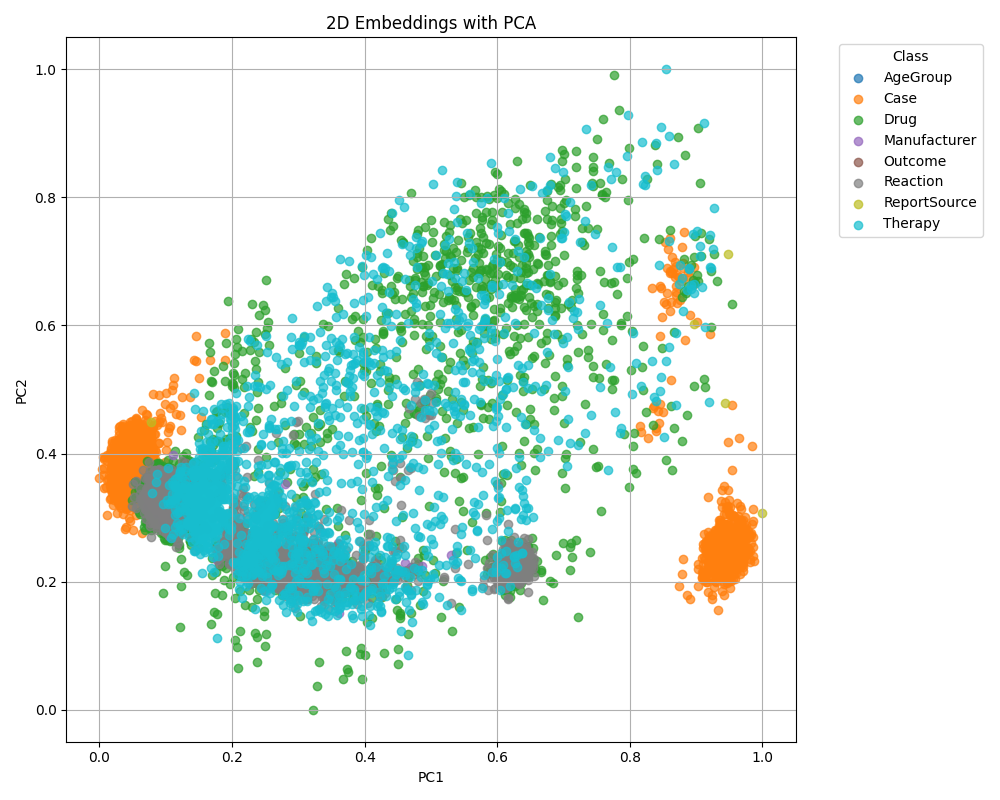}
        \caption{Node2Vec}
        \label{fig:node2vec_embeddings}
    \end{minipage}
    \begin{minipage}[b]{0.32\linewidth}
        \centering
        \includegraphics[width=\linewidth]{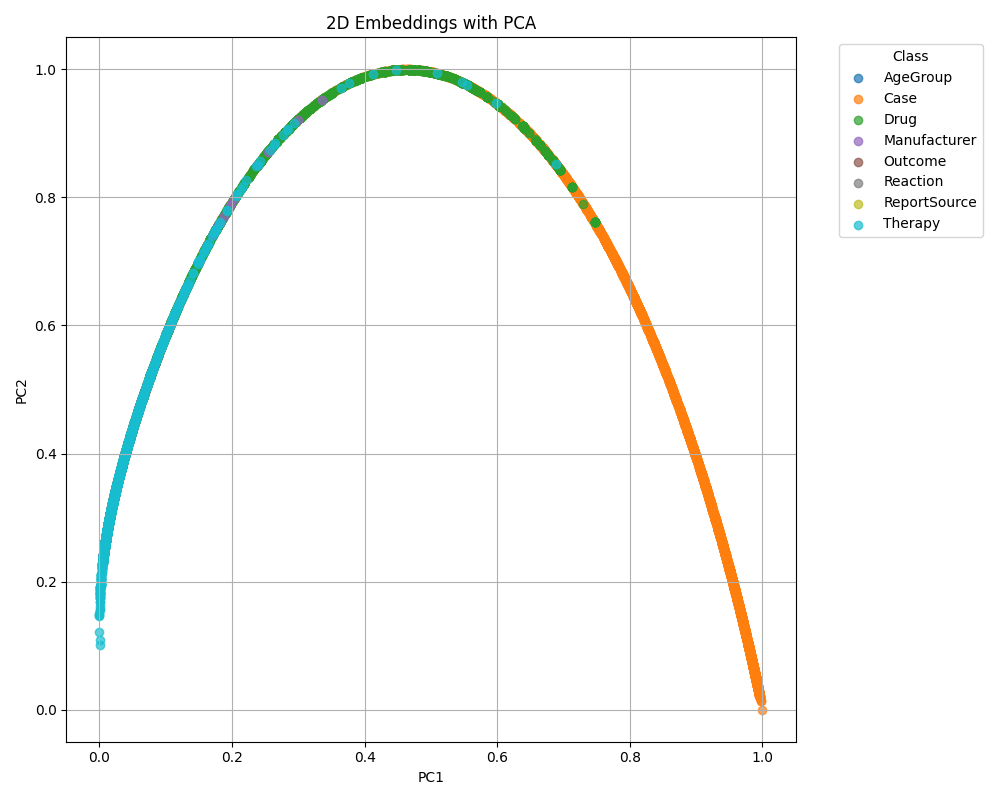}
        \caption{GraphSAGE}
        \label{fig:graphsage_embeddings}
    \end{minipage}
    \begin{minipage}[b]{0.32\linewidth}
        \centering
        \includegraphics[width=\linewidth]{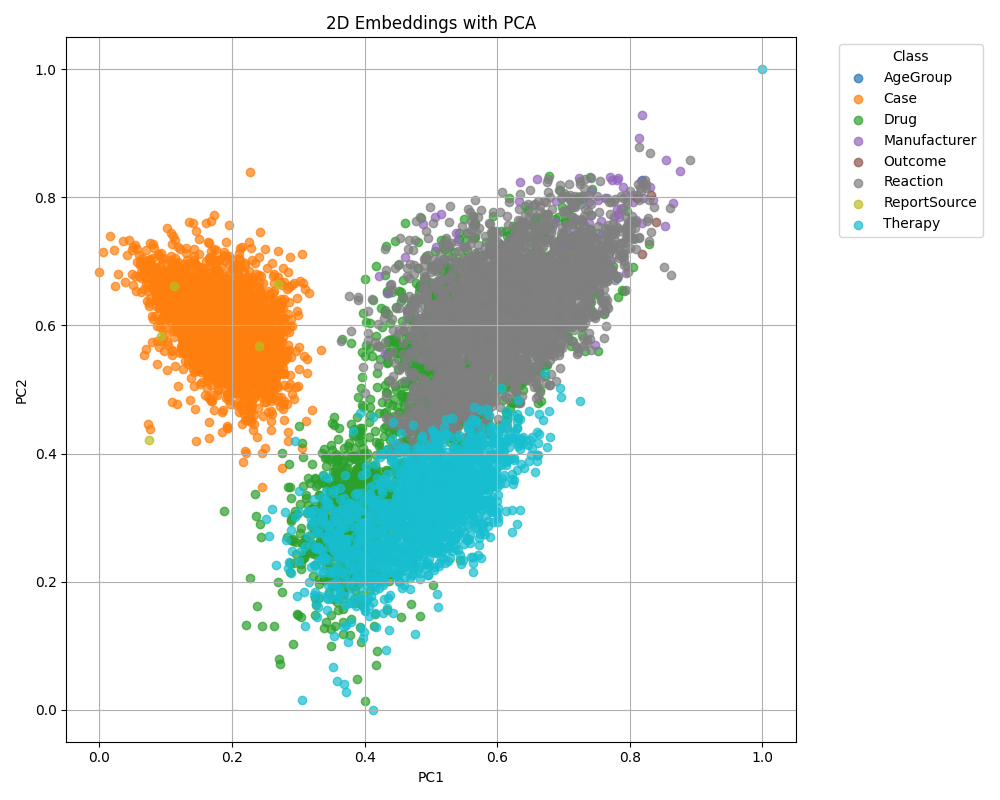}
        \caption{Bi-View}
        \label{fig:enhanced}
    \end{minipage}
    
    \vspace{1em} 

    \begin{minipage}[b]{0.49\textwidth}
        \centering
        \includegraphics[width=\linewidth]{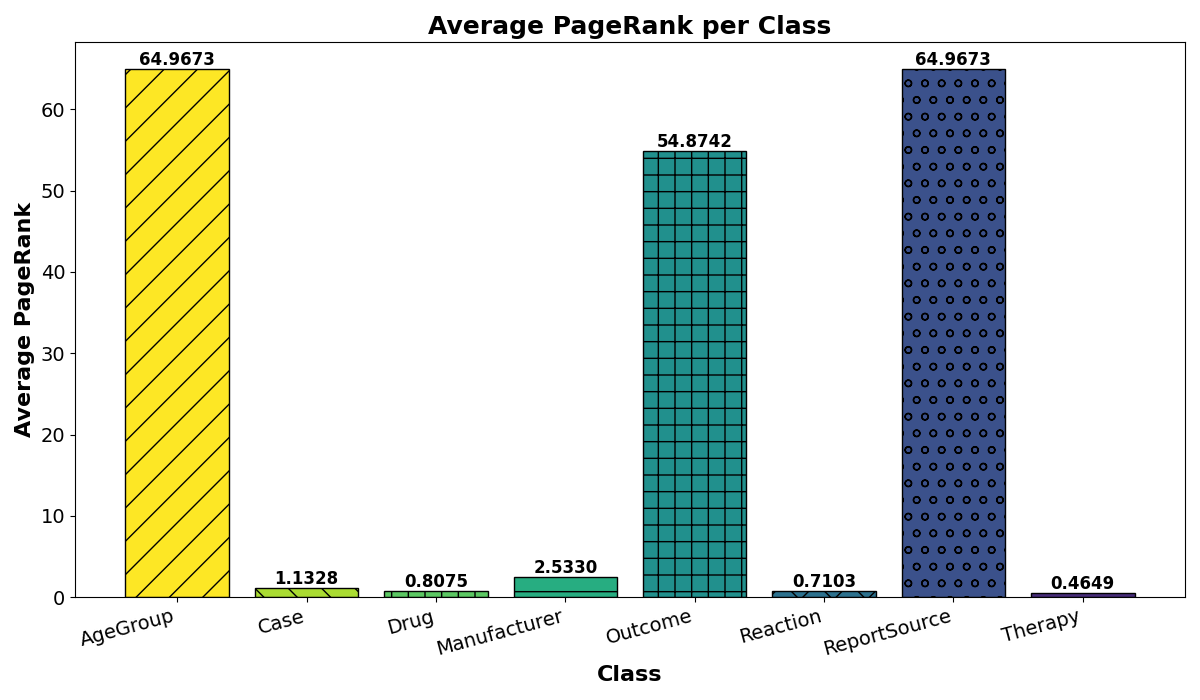}
        \caption{Average PageRank per Class.}
        \label{fig:avg-pagerank}
    \end{minipage}
    \begin{minipage}[b]{0.49\textwidth}
        \centering
        \includegraphics[width=\linewidth]{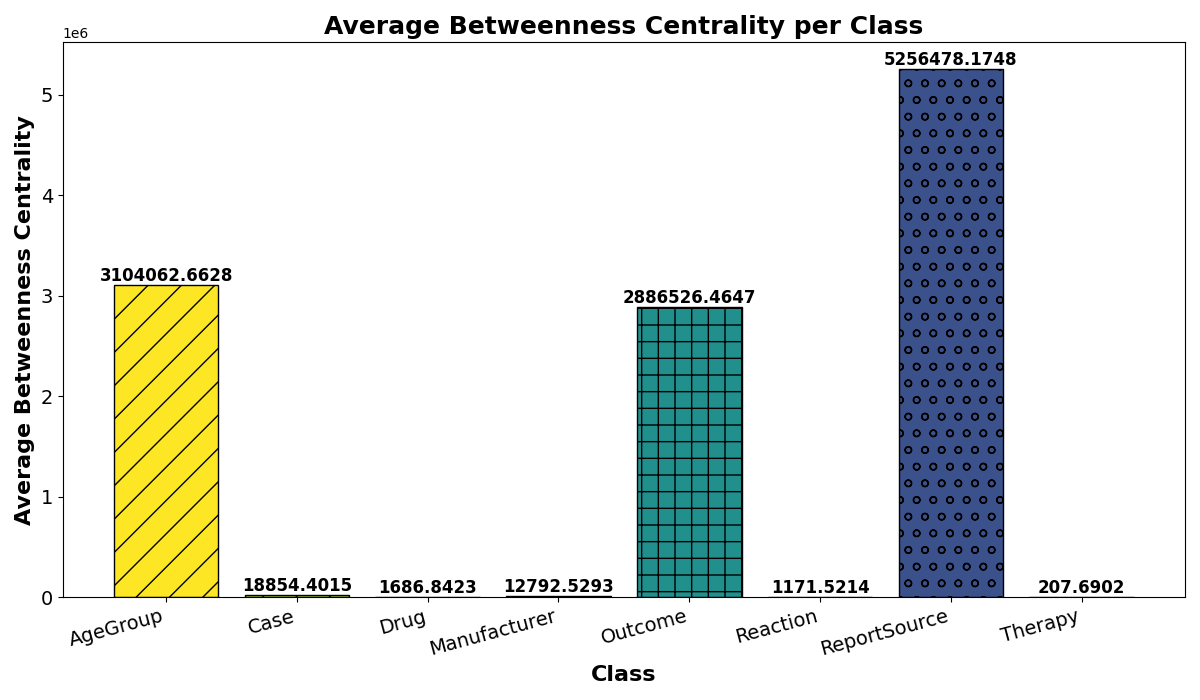}
        \caption{Average Betweenness per Class.}
        \label{fig:avg-betweenness}
    \end{minipage}
\end{figure}
Given the highly imbalanced nature of the dataset (Fig. \ref{fig:class_distribution}) where class samples show differences of up to three orders of magnitude, the evaluation has been limited to the four classes with the highest number of instances: Case, Drug, Reaction and Therapy. This restriction enables a more robust comparison across the Bi-View, GraphSAGE, and Node2Vec models, avoiding distortions caused by extremely rare classes.
To better justify our choice, node count alone does not fully capture the structural importance of each class within the graph. For these reasons, we also obtain insights drawn from the plots of average PageRank (Fig. \ref{fig:avg-pagerank}) and average betweenness (Fig. \ref{fig:avg-betweenness}) centrality, which reveal that the underrepresented classes exhibit exceptionally high centrality scores, indicating that their nodes act as central hubs that are basically in common to the other classes. As a result, using these highly central nodes as features for classification could introduce noise rather than discriminative power, since their presence does not provide class-specific information but reflects a global network structure shared across all classes, negatively affecting the model's ability to distinguish between classes and reducing its overall performance, thus motivating our choice to ignore them for classification purpose. 

We calculated Accuracy, Precision, Recall and F1-Score to evaluate the performance of the three approaches, in particular: Bi-View, Node2Vec and GraphSAGE on classes with at least 200 samples. Bi-View (Fig. \ref{fig:metrics} ) outperforms the other methods for all metrics, reaching an accuracy close to 94\% with high precision and recall, leading to a strong F1-score. Node2Vec shows intermediate performance, with scores consistently higher than GraphSAGE but still markedly below Bi-View. On the contrary, GraphSAGE reports the lowest values in all metrics, with a particularly reduced F1-score of around 64\%. These results suggest that Bi-View is more effective at capturing patterns with class imbalance, even when the evaluation is restricted to the most statistically significant classes. In the following analysis, we examine the confusion matrices (Fig. \ref{fig:node2vec}, Fig. \ref{fig:graphsage} and Fig. \ref{fig:enhanced}) for a deeper understanding of how each model behaves on these dominant classes.

\begin{figure}[ht!]
    \centering
    \includegraphics[width=0.85\linewidth]{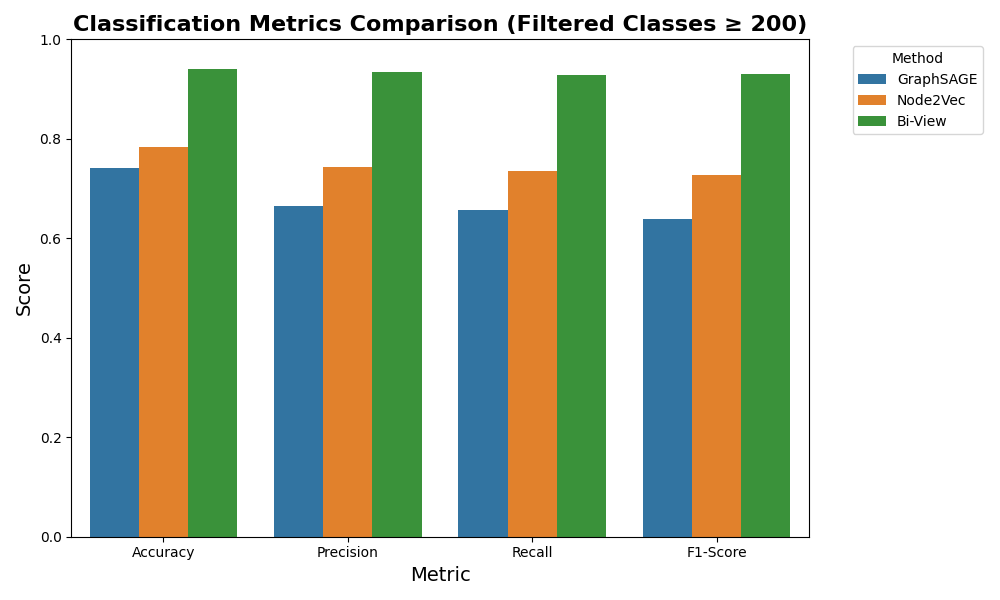}
    \caption{Models Comparison.}
    \label{fig:metrics}
\end{figure}
\vspace{-3.4 em}  
\begin{figure}[ht]
    \centering
    \begin{minipage}[b]{0.325\linewidth}
        \centering
        \includegraphics[width=\linewidth]{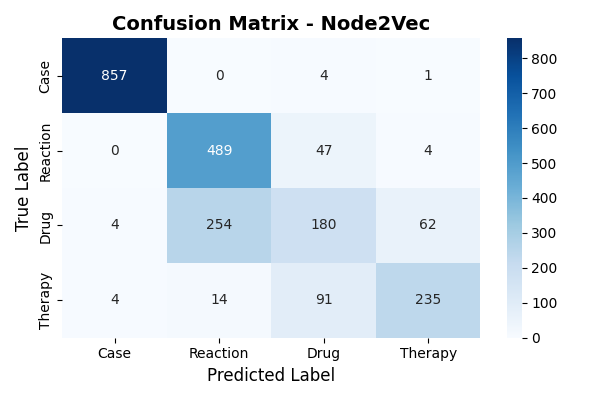}
        \caption{Node2Vec}
        \label{fig:node2vec}
    \end{minipage}
    \begin{minipage}[b]{0.325\linewidth}
        \centering
        \includegraphics[width=\linewidth]{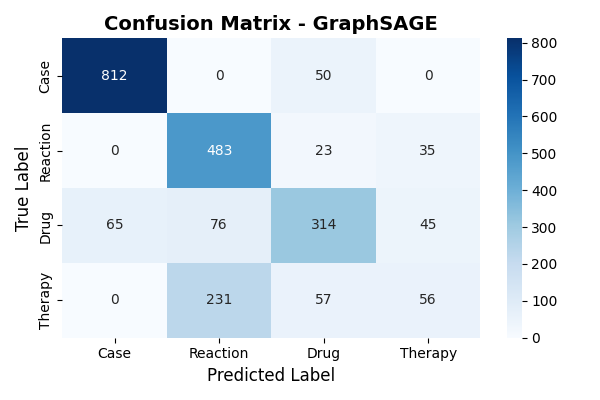}
        \caption{GraphSAGE}
        \label{fig:graphsage}
    \end{minipage}
    \hfill
    \begin{minipage}[b]{0.325\linewidth}
        \centering
        \includegraphics[width=\linewidth]{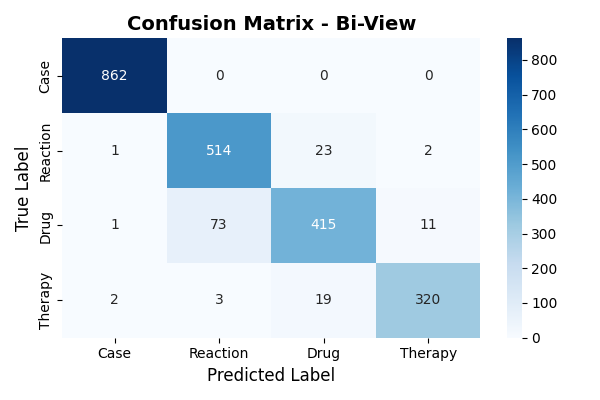}
        \caption{Bi-View}
        \label{fig:enhanced}
    \end{minipage}
\end{figure}

Bi-View performs better than both Node2Vec and GraphSAGE in every category, especially when it comes to differentiating between closely related classifications like Drug and Therapy, according to the confusion matrices shown in Figures \ref{fig:node2vec}, \ref{fig:graphsage} and \ref{fig:enhanced}.  Node2Vec performs well for the Case and Reaction classes, but it has trouble with a high rate of drug and reaction misclassification.  Though it still has significant confusion, particularly between Therapy and Reaction, GraphSAGE offers more balanced forecasts.  Bi-View, on the other hand, has the best classification accuracy and the fewest cross-category mistakes, demonstrating its exceptional capacity to gather and incorporate a variety of relational and structural data for better class separability.

\section{Conclusion and Future Work}
\label{sec:conclusion}
In this study, we proposed a hybrid graph embedding strategy that integrates unsupervised structural representations from Node2Vec with supervised neighbourhood aggregation from GraphSAGE through a dynamic fusion mechanism. This Bi-View Embedding Fusion strategy was specifically designed to address the limitations of traditional GML models in poor-dataset feature scenarios. By enriching initial node features with topological insights and centrality metrics and adaptively combining global and local representations, our approach produces more expressive and semantically informative embeddings. Experimental results on the FAERS healthcare dataset confirmed the effectiveness of the proposed method, showing improvements in classification tasks under low-features regimes.
These findings demonstrate that our fusion model is not merely a feature concatenation scheme but a transformative strategy that enhances generalisation by capturing diverse and complementary patterns within the graph.
As future work, we plan to evaluate the performance of this strategy across different types of graphs, including temporal and multi-relational datasets. Additionally, we intend to explore meta-learning and inductive transfer capabilities to further improve adaptability in dynamic or cross-domain graph environments.

\section*{Acknowledgments}
Rosario Napoli is a PhD student enrolled in the National PhD in Artificial Intelligence, XL cycle, course on Health and Life Sciences. Giovanni Lonia is a PhD student enrolled in the National PhD in Artificial Intelligence, XL cycle, course on AI for society, organised by the University of Pisa.
This work has been partially funded by the the Italian Ministry of Health, Piano Operativo Salute (POS) trajectory 4 “Biotechnology, bioinformatics and pharmaceutical development”, through the Pharma-HUB Project "Hub for the repositioning of drugs in rare diseases of the nervous system in children" (CUP J43C22000500006), the Italian Ministry of Health, Piano Operativo Salute (POS) trajectory 2 “eHealth, diagnostica avanzata, medical device e mini invasività” through the project ``Rete eHealth: AI e strumenti ICT Innovativi orientati alla Diagnostica Digitale (RAIDD)''(CUP J43C22000380001), the “SEcurity and RIghts in the CyberSpace (SERICS)” partnership (PE00000014), under the MUR National Recovery and Resilience Plan funded by the European Union – NextGenerationEU. In particular, it has been supported within the SERICS partnership through the projects FF4ALL (CUP D43C22003050001) and SOP (CUP H73C22000890001), the Horizon Europe NEUROKIT2E project (Grant Agreement 101112268).

\bibliographystyle{unsrt}
\bibliography{bib}
\end{document}